\documentclass[copyright,creativecommons]{eptcs}
\providecommand{\event}{TARK 2023} 

\usepackage[utf8]{inputenc}
\usepackage{amsmath}
\usepackage{amssymb}
\usepackage{amsthm}
\usepackage{stmaryrd}
\usepackage{tikz}
\usepackage{enumerate}
\usepackage{soul}
\usepackage{centernot}
\usepackage{enumitem}
\usepackage{xcolor}
\usepackage{multirow}

\usepackage{amsfonts}
\DeclareFontFamily{U}{mathb}{\hyphenchar\font45}
\DeclareFontShape{U}{mathb}{m}{n}{
      <5> <6> <7> <8> <9> <10> gen * mathb
      <10.95> mathb10 <12> <14.4> <17.28> <20.74> <24.88> mathb12
}{}
\DeclareSymbolFont{mathb}{U}{mathb}{m}{n}
\DeclareMathSymbol{\llcurly}{3}{mathb}{"CE}
\DeclareMathSymbol{\ggcurly}{3}{mathb}{"CF}

\usepackage{graphicx}
\graphicspath{ {./} }
\usepackage{subcaption}

\newcommand{\nc}{\,\mid\!\sim\,}

\usepackage{amsmath}
\usepackage{amsthm}

\newtheorem*{theorem*}{Theorem}
\newtheorem*{conjecture*}{Conjecture}

\newtheorem{proposition}{Proposition}

\theoremstyle{definition}

\theoremstyle{definition}
\newtheorem{definition}{Definition}[section]
\theoremstyle{definition}

\theoremstyle{definition}

\theoremstyle{definition}

\begin{document}

\title{Belief Revision from Probability}
\author{Jeremy Goodman
\institute{School of Philosophy\\University of Southern California, USA}
\email{jeremy.goodman@usc.edu}
\and
Bernhard Salow
\institute{Faculty of Philosophy\\University of Oxford, UK}
\email{bernhard.salow@philosophy.ox.ac.uk}
}
\def\titlerunning{Belief Revision from Probability}
\def\authorrunning{J. Goodman \& B. Salow}

\maketitle

\begin{abstract}

In previous work (\cite{GoodSaloTARK, GoodSaloEpis}), we develop a question-relative, probabilistic account of belief. On this account, what someone believes relative to a given question is (i) closed under entailment, (ii) sufficiently probable given their evidence, and (iii) sensitive to the relative probabilities of the answers to the question.  Here we explore the implications of this account for the dynamics of belief. We show that the principles it validates are much weaker than those of orthodox theories of belief revision like AGM \cite{AGM85}, but still stronger than those valid according to the popular Lockean theory of belief \cite{Fole93}, which equates belief with high subjective probability. We then consider a restricted class of models, suitable for many but not all applications, and identify some further natural principles valid on this class. We conclude by arguing that the present framework compares favorably to the rival probabilistic accounts of belief developed by Leitgeb \cite{Leit14, Leit17} and Lin and Kelly \cite{LinKell12}.

\end{abstract}

\section{Probability Structures}

We will work with the following simplification of the models in \cite{GoodSaloTARK}:
\begin{definition} A \emph{probability structure} is a tuple $\langle S,\mathcal{E},Q,Pr,t\rangle$ such that:
\begin{enumerate}
    \item $S$ is a non-empty set (of \emph{states}),
    \item $\mathcal{E}\subseteq\mathcal{P}(S)\backslash \{\emptyset\}$ (the \emph{possible bodies of evidence}),
    \item $Q$ (the \emph{question}) is a partition of $S$,
    \item $Pr$ (the \emph{prior}) is a probability distribution over $S$, and
    \item $t\in [0,1]$ (the \emph{threshold})
\end{enumerate}
\end{definition}
\noindent Propositions are modeled as subsets of $S$, where $p$ is true in $s$ if and only if $s\in p$. We say that $E'\in \mathcal{E}$ is \textit{the result of discovering} $p$ in $E\in \mathcal{E}$ just in case $E'=E\cap p$; this will allow us to talk about how beliefs evolve in response to changes in one's evidence.

Which propositions an agent believes is a function of their evidence and is also given by a set of states, so that an agent with evidence $E$ believes $p$ if and only if $B(E)\subseteq p$. This ensures that their beliefs are closed under entailment, and thus already marks a departure from popular `Lockean' accounts of belief \cite{Fole93}, according to which one believes a proposition if and only if its probability exceeds a particular threshold. But it is compatible with the more plausible direction of Lockeanism, namely:
\begin{itemize}
    \item[] \textsc{threshold}: You believe $p$ only if $p$ is sufficiently probable given your evidence.
    \subitem If $B(E)\subseteq p$, then $Pr(p|E)\geq t$.
\end{itemize}

We can think of the members of the question $Q$ as its \textit{answers}; we write $[s]_Q$ for the member of $Q$ containing $s$. The proposal in \cite{GoodSaloTARK} then boils down to claiming that $s\in B(E)$ if and only if $s\in E$ and the answers to $Q$ that are more probable than $[s]_Q$ have total probability less than the threshold $t$. Writing $Pr_E$ for $Pr(\cdot |E)$, this can be formalized as follows: 

\begin{definition}
$B(E)=\{s\in E: Pr_E(\{s': Pr_E([s']_Q) > Pr_E([s]_Q)\})<t\}$
\end{definition}

\noindent This means that one believes as much as possible subject to two constraints: (i) \textsc{threshold}, and (ii) that the totality of one's beliefs corresponds to the conjunction of one's evidence with a disjunction of answers to $Q$ that includes any answer at least as probable (given one's evidence) as any other it includes. 
One notable attraction of this proposal is that what one believes corresponds to the discrete analogue of the highest posterior-density region typically used to define `credible intervals' from probability density functions in Bayesian statistics.  A logically significant feature of the proposal, to which we will return later, is that it involves not only local probability comparisons between different answers to $Q$, but also a global probability comparison between a collection of such answers and the threshold $t$.



\section{Principles and Results}

A core idea behind the orthodox AGM \cite{AGM85} theory of belief revision is that belief revisions are trivial whenever what you learn is compatible with your initial beliefs: you should simply add the discovery to your beliefs, draw out the logical consequences of these beliefs, and leave everything else unchanged. Here we will focus on five principles that encode various aspects of this idea. Exploring when and how these principles can fail will be a useful way of exploring the extent to which our account of belief requires departing from orthodoxy when it comes to belief dynamics. 
These principles are:\footnote{\label{fn:names}The $\Diamond$ indicates that the discovery is compatible with your initial beliefs, while the $\Box$ indicate that it is something you initially believe. $\Diamond -$ is often referred to as `preservation'; \cite{SheaFite19} call $\Box -$ `weak preservation' and $\Box R$ `very weak preservation'. If we interpret the non-monotonic consequence relation $p\nc q$ as saying that $B(p)\subseteq q$, then $\Diamond -$ corresponds to `rational monotony',  $\Box +$ to `cut', and $\Box -$ to `cautious monotony' in the standard terminology from \cite{KrauLehmMagi90}.}


\begin{itemize}
\item[$\Diamond -$] If you don't believe not-$p$ and then discover $p$, you shouldn't give up any beliefs.
\subitem If $B(E)\cap p \neq \emptyset $, then $B(E\cap p) \subseteq B(E)$.
\item[$\Diamond R$] If you don't believe not-$p$ and then discover $p$, you shouldn't reverse any of your beliefs (i.e.  go from believing something to believing its negation). 
\subitem If $B(E)\cap p \neq \emptyset $, then $B(E)\cap B(E\cap p) \neq \emptyset$.
\item[$\Box +$] If you believe $p$ and then discover $p$, you shouldn't form any new beliefs.
\subitem If $B(E)\subseteq p $, then $B(E) \subseteq B(E\cap p)$.
\item[$\Box -$] If you believe $p$ and then discover $p$, you shouldn't give up any beliefs.
\subitem If $B(E)\subseteq p $, then $B(E\cap p) \subseteq B(E)$
\item[$\Box R$] If you believe $p$ and then discover $p$, you shouldn't reverse any of your beliefs.
\subitem If $B(E)\subseteq p $, then  $B(E) \cap B(E\cap p) \neq \emptyset$.
\end{itemize}
These principles are not logically independent: the $\Diamond$ principles entail the corresponding $\Box$ principles, and the $+$ and $-$ principles each entail the corresponding $R$ principles. All of them are valid according to AGM. By contrast, only $\Box R$ is valid according to Lockean theories that equate believing a proposition with assigning it a sufficiently high probability (for some probability threshold less than 1), and it is valid only if this probability threshold is above $\frac{\sqrt{5}-1}{2}\approx .62$ (as discussed in \cite{SheaFite19}). 

The present account falls in between these extremes:
\begin{proposition}\label{prop:probval}
$\Box -$ and $\Box R$ are valid on the class of probability structures. 
\end{proposition}

\begin{proposition}\label{thm:probinval}
$\Diamond -$, $\Diamond R$, and $\Box +$ can all fail in  probability structures.
\end{proposition}

We will illustrate Proposition~\ref{thm:probinval} with two examples. Consider a much discussed thought experiment: 
\begin{quote} \textbf{Flipping for Heads}\\A coin flipper will flip a fair coin until it lands heads. 
\end{quote}
A natural model of this case is as follows:

\begin{quote}
\begin{tabular}{ll}
  $S=\{s_1,s_2,\ldots\}$   &  $\mathcal{E}=\{\{s_i,s_{i+1},s_{i+2},\dots\}: s_i\in S\}$\\
   $Q=\{\{s_i\}:s_i\in S\}$  & $Pr(\{s_i\})=\frac{1}{2^i}$\\
   $t=.99$ &
\end{tabular}
\end{quote}
Here $s_i$ is the state in which the coin lands heads on the $i$th flip, and 
$\{s_i,s_{i+1},s_{i+2},\dots\}$ is your evidence if you have just observed the coin land tails on the first $i-1$ flips. The question $Q$ is maximally fine-grained, and the probabilities match the known objective chances.

In this model, $B(\{s_i,s_{i+1},s_{i+2},\dots\})=\{s_i, s_{i+1},\ldots, s_{i+6}\}$: you always believe that the coin will land heads within the next seven flips. $\Diamond -$ is violated whenever you observe the coin land tails. For example, let $p=\{s_2,s_3,\dots\}$. Then $B(S)\cap p\neq\emptyset$, but $B(S\cap p)=\{s_2,\dots,s_8\}\not\subseteq\{s_1,\dots,s_7\}=B(S)$. We think this is exactly the right prediction.

To turn this into a counterexample to $\Box +$, we add new body of evidence $E'=\{s_1, s_2, \ldots, s_7\}$ to $\mathcal{E}$. Intuitively, we can think of this as the evidence you receive if you walk away from the experiment before the first flip, and are later told that the coin landed heads within the first seven flips. It is easy to verify that $B(E')=\{s_1,\dots,s_6\}$. So $B(S)\not \subseteq B(S\cap E')$, even though $B(S) \subseteq E'$. That this can happen should be unsurprising in a framework like ours in which agents have `inductive' beliefs that go beyond what is strictly entailed by their evidence: discovering something that you previously believed only inductively will strengthen your evidence, putting you in a position to draw further inductive conclusions.  

Counterexamples to $\Diamond R$ are subtler, for reasons we will explain in the next section. But here is one:
\begin{quote}\textbf{Drawing a Card}\newline You are holding a deck of cards, which is either a fair deck consisting of 52 different cards or a trick deck consisting of 52 Aces of Spades. Your background evidence makes it 90\% likely that the deck is fair. You draw a card at random; it is an Ace of Spades.
\end{quote}
Here is a possible model of the example:
\begin{quote}
\begin{tabular}{ll}
  $S=\{F_1,F_2,\ldots, F_{52}, T\}$   &  $\mathcal{E}=\{S, \{F_1\},\ldots,\{F_{51}\},\{F_{52},T\}\}$\\
   $Q=\{\{F_1,F_2,\ldots, F_{52}\},\{T\}\}$  & $Pr(\{F_i\})=\frac{.9}{52}\approx.017, Pr(\{T\})=.1$\\
   $t=.85$ &
\end{tabular}
\end{quote}

The states $F_i$ are all states in which the deck is fair; they are distinguished only by which card you will draw, with $F_{52}$ being the one where you draw the Ace of Spades. State $T$ is the state in which the deck is the trick deck (and you thus draw an Ace of Spades). Your evidence settles all and only what card you drew; so when you draw an Ace of Spades, it leaves open both that you did so by chance and that you did so because it is a trick deck. The question is simply whether the deck is fair. It is easy to see that, according to this model, you should initially believe only that the deck is fair. Your initial beliefs are thus compatible with it being fair and you drawing the Ace of Spades by chance. Yet when you discover that you drew an Ace of Spades, you should reverse your opinion and conclude that you're holding the trick deck, since $Pr(\{T\}|\{F_{52},T\})\approx \frac{.1}{.1+.017}\approx .855>t$.

Note that, in this model, your discovery is not a disjunction of answers to the question $Q$. If we changed the question to a more fine-grained one, so that your discovery was a disjunction of its answers, then the case would no longer yield a counterexample to $\Diamond R$. For example, relative to the question \textit{is the deck fair and will I draw an Ace of Spades} -- i.e. relative to $Q'=\{\{F_1,F_2,\ldots, F_{51}\},  \{F_{52}\},\{T\}\}$ -- you will initially believe that you won't draw an Ace of Spades, in which case your subsequent discovery isn't compatible with your initial beliefs. And relative to the question \textit{is the deck fair and what will I draw} -- i.e. relative to the maximally fine-grained $Q''=\{\{s\}:s\in S\}$ -- you will initially have no non-trivial beliefs, and in particular you won't start out believing that the deck is fair. In the next section, we will see that this is part of a more general pattern about counterexamples to $\Diamond R$. 

\section{Orthogonality}

In the previous section, we saw that some of the surprising belief dynamics in probability structures depended on discoveries that cross-cut the question. Notice that structures in which this cannot happen, because every member of $\mathcal{E}$ is the union of some subset of $Q$, satisfy the following constraint: 
\begin{itemize}
    \item[] \textsc{orthogonality}: $\frac{Pr([s]_Q)}{Pr([s']_Q)}=\frac{Pr([s]_Q|E)}{Pr([s']_Q|E)}$ for all $s,s'\in E\in\mathcal{E}$ s.t.  $Pr([s']_Q|E)>0$ 
\end{itemize}
This says that the only way that getting new evidence can change the relative probability of two answers to $Q$ is by completely ruling out one of those answers. While we can ensure \textsc{orthogonality} by making the question fine-grained enough to capture all possible discoveries, this isn't always necessary. For example, we could fine-grain the states and bodies of evidence in our model of \textbf{Flipping for Heads} to capture the fact that you discover where on the table the coin lands. The bodies of evidence in such a fine-grained model will cross-cut the question \emph{how many time will the coin be flipped}; but, plausibly, \textsc{orthogonality} will still hold for this question, since the added information about where the coin lands is probabilistically independent from how many times it will be flipped.

\textsc{orthogonality} is interesting because it leads to a stronger logic of belief revision. Firstly,
\begin{proposition}
$\Diamond R$ is valid on the class of probability structures satisfying \textsc{orthogonality}. 
\end{proposition}
\noindent Secondly, consider the following principle. It says (roughly) that if you're sure that, whatever you're about to discover, you won't believe a given proposition afterwards, then you already don't believe it:
\begin{itemize}
    \item[$\Pi -$] If $\Pi$ is a partition any member of which you could discover, then there is a $p\in \Pi$ such that you shouldn't give up any beliefs upon discovering $p$.
    \begin{itemize}
        \item[] If $\Pi\subseteq\mathcal{E}$ is a partition of $E$, then $B(E\cap p)\subseteq B(E)$ for some $p\in \Pi$.
    \end{itemize} 
    \end{itemize}
We then have the following result: 
\begin{proposition}
$\Pi -$ can fail in probability structures. But it is valid on the class of probability structures satisfying \textsc{orthogonality}.
\end{proposition}

It is also worth noting that $\Diamond -$ and $\Box +$ can still fail in structures satisfying \textsc{orthogonality}. In particular, \textsc{orthogonality} holds in the structures we used in the last section to argue that \textbf{Flipping for Heads} yields counterexamples to $\Diamond -$ and $\Box +$.

In our view, a good deal of ordinary talk about what people believe is well-modelled by structures satisfying \textsc{orthogonality}. This is because we think that the question $Q$, to which attributions of belief are implicitly relativized, typically coincides with the question under discussion in the conversational context in which those attributions are made. Moreover, when a discovery is salient, it is natural to consider a question that is sufficiently fine-grained to capture all the aspects of this discovery that are relevant to its answers. Counterexamples to \textsc{orthogonality} (and thus to $\Diamond R$ and $\Pi -$) therefore tend to be `elusive' in Lewis's \cite{LEwi96} sense: attending to these cases often changes the context in such a way that they can no longer be described as counterexamples. 

That being said, we do not think that \textsc{orthogonality} is plausible as a general constraint. This is because, very often, the only way to ensure \textsc{orthogonality} is to adopt a very fine-grained question; and, often, such fine-grained questions make overly skeptical predictions about what we can believe. Consider, for example, the following case:
\begin{quote}
    \textbf{One Hundred Flips}\\
    You will flip a fair coin 100 times and watch how it lands each time. 
\end{quote}
There are natural contexts in which you can be correctly described as initially believing that the coin will not land heads more than 90 times. Our theory predicts this for various natural questions, even for very high thresholds $t$ -- for example, the polar question \textit{will the coin land heads more than 90 times} or the slightly more fine-grained question \textit{how often will the coin lands heads}. But neither of these questions satisfies \textsc{orthogonality}. For example, discovering that coin lands tails on the first flip will favor `no' over `yes' for the first question, and `51' over `49' for the second question, without ruling out any of these answers. In fact, the only natural question that satisfies \textsc{orthogonality} is the maximally fine-grained question \textit{what will the exact sequence of heads and tails be}. But all answers to this question are equally likely, and so this question prevents you from having any non-trivial beliefs about what will happen.

We conclude that \textsc{orthogonality} should be rejected as a general constraint, even if it will often hold when we are considering a particular case with a limited range of discoveries. $\Diamond R$ and $\Pi -$ are thus not fully general principles of belief revision; but counterexamples are likely to be difficult to pin down. 

\textsc{orthogonality} is also a fruitful principle in that it helps to facilitate comparisons between our framework and other probabilistic theories of belief. Let us now turn to these.

\section{Comparisons}

In this section, we consider two influential probabilistic accounts from the literature, and compare them with our own account. The first can be seen as a version of our account with an additional constraint imposed on probability structures, and validates $\Diamond -$ but not $\Box+$; the second can be seen as defining belief from probability structures in a related but different way, and validates $\Box+$ but not $\Diamond -$.

\subsection{A Stability Theory}

The first theory we want to consider is inspired by Leitgeb's \emph{stability theory of belief} \cite{Leit14, Leit17}. The guiding idea behind this theory is a probabilistic analogue of $\Diamond-$ that Leitgeb calls the \emph{Humean thesis}. However, despite the `stablity' moniker, the constraints imposed by Leitgeb's theory are \emph{synchronic} ones relating probabilities, partitions, and thresholds at a single time. So both to facilitate comparison with our framework, and to be (in our view) more faithful to its motivating idea, we will consider a strengthening of Leitgeb's theory according to which the requirements it imposes on one's beliefs prior to a discovery continue to hold after one has made that discovery. We can then interpret the view as proposing the following constraint on probability structures:\footnote{Our \textsc{stability} strengthens Leitgeb's theory in two ways: first, by identifying the threshold that characterizes the minimal probability of anything one believes with the threshold in terms of which stability is defined, and, second, by not allowing this threshold to be different for different possible bodies of evidence. It also departs from his formulation in quantifying over $\mathcal{E}$ rather than $\{\bigcup Y: Y\subseteq Q\}$; however, we read him as identifying $\mathcal{E}$ with $\{\bigcup Y: \emptyset\neq Y\subseteq Q\}$, so this is not a substantive departure.} 

\begin{itemize}
    \item[] \textsc{stability}: For all $E\in\mathcal{E}$ and $X\subseteq Q$, if $Pr(\bigcup X)\geq t$ and $E\cap\bigcup X \neq \emptyset$, then $Pr(\bigcup X|E)\geq t$. 
\end{itemize}
We then have the following results:

\begin{proposition} $\Diamond -$ is valid in probability structures satisfying \textsc{stability} and \textsc{orthogonality}. But $\Box+$ can still fail; and $\Diamond -$ can fail in structures satisfying \textsc{stability} but not \textsc{orthogonality}.
\end{proposition}
    
\noindent This illustrates how a kind of qualitative stability of belief can be secured by a kind of probabilistic stability (given \textsc{orthogonality}), without entailing the full strength of AGM.\footnote{Leitgeb \cite[chapter~4]{Leit17} describes his theory as compatible with AGM (and thus with $\Box +$) since, upon getting new evidence, one may adopt a different, higher threshold than before. But doing so is in no way required by the demands of stability.
} 

We reject \textsc{stability} because we reject $\Diamond -$ (even in cases where \textsc{orthogonality} holds), and along with it the informal idea that rational belief should be stable in anything like the way that Leitgeb claims it should be. \textsc{stability} also places implausible constraints on what agents can believe at a given time. For example, \cite{KellLin21} show, in effect, that in \textbf{Flipping for Heads} \textsc{stability} entails that the only way to have any non-trivial beliefs about how many times the coin will be flipped is to believe that it will be flipped only once. (This argument depends only on the symmetries of the example, and doesn't depend on whether the coin is fair, biased towards heads, or biased towards tails.) See also \cite{Rott17} and \cite{DouvRott18}. 

\subsection{The Tracking Theory}\label{sec:LK}

Lin and Kelly \cite{LinKell12} defend a theory which (for reasons we can't explain here) they call the `tracking theory' of belief. 
This theory can be seen as an alternative way of defining belief in probability structures, with the parameter $t$ playing a rather different role. Put informally, a state $s$ is compatible with your LK-beliefs if there is no answer to $Q$ that is more than $\frac{1}{t}$ times more likely than $[s]_Q$. Formally:
\begin{definition} \label{def:LK}
$B_{LK}(E)=\{s\in E: (\forall q \in Q) (Pr_E([s]_Q) \geq t\times Pr_E(q))\}$
\end{definition}

\noindent In many cases -- such as \textbf{Flipping for Heads} -- the subject will have similar beliefs according to our theory and according to Lin and Kelly's (provided $t$ is chosen judiciously: low values of $t$ for Lin and Kelly correspond to high values of $t$ for us). However, there are important structural differences between the theories. In particular, LK-beliefs are sensitive only to local comparisons of probability between particular answers, while beliefs as we understand them depend also the probabilities of sets of answers. A consequence of this locality is that, as Lin and Kelly note, their theory validates a reasonably strong theory of belief revision (assuming \textsc{orthogonality}, which they essentially build in):

\begin{proposition} $\Box +$, $\Box -$, $\Box R$, $\Diamond R$,  and $\Pi -$ are all valid for LK-belief on the class of probability structures satisfying \textsc{orthogonality}. 
\end{proposition}

The major shortcoming of the tracking theory, in our view, is that it fails to entail \textsc{threshold}. Consider a case like \textbf{Drawing a Card}, in which one state initially has very low probability ($.1$) but every other state has even lower probability ($.017$). Then relative to a fine-grained question such as \textit{is the deck fair and which card will you draw}, you will LK-believe that the deck is a trick deck even for reasonably low values of $t$ (such as $.2$). But this belief is only $.1$ likely on your evidence! And we can, of course, make the case more extreme by increasing the number of distinct cards in the fair deck; so the believed proposition can be arbitrarily improbable for any fixed value of $t$.


One might defend the tracking theory against such cases by insisting that we choose a more coarse-grained question; while the theory still fails to entail \textsc{threshold}, this response at least prevents it from recommending the extreme violations just discussed. However, moving to coarser-grained questions is often in conflict with \textsc{orthogonality}. Moreover, the reasons we gave previously for rejecting \textsc{orthogonality} as a general constraint applies to the tracking theory as well: just like our theory, the tracking theory will make implausibly skeptical predictions in \textbf{One Hundred Flips} unless combined with an \textsc{orthogonality}-violating question such as \textit{how many heads will there be}.

Without \textsc{orthogonality}, the dynamics of LK-belief are substantially less constrained:
\begin{proposition} $\Box +$ and $\Box -$ are valid for LK-belief on the class probability structures.  $\Diamond R$ and $\Pi -$ can both fail in such structures.
\end{proposition}
\noindent $\Box +$ is then the only principle valid for LK-beliefs but not for beliefs as we understand them.

Moreover, without \textsc{orthogonality}, the tracking theory invalidates a new principle that holds for belief as we understand it (assuming we restrict to probability structures with $t>.5$). Consider the following variant of \textbf{Drawing a Card} (taken from \cite{Hack67}, who also makes parallel observations as an objection to Levi's \cite{Levi67} account of belief):
\begin{quote}
    \textbf{Drawing a Card v.2}\\
    You are holding a deck which could be either a `fair' deck of 52 different cards, or one of 52 different `trick' decks that just contain the same card 52 times. Given your background evidence, the probability that you are holding the fair deck is $\frac{1}{5}$, with the remaining $\frac{4}{5}$ distributed evenly across the 52 trick decks. You are about to draw and turn over one card from your deck. 
\end{quote}
Let us assume that $Q$ is \emph{which of the 53 possible decks am I holding} and $t>.25$. According to the tracking theory, you initially believe that you hold the fair deck, but after drawing a card you believe that you are holding the relevant trick deck. So we have a failure of the following principle, which says (roughly) that if you're sure that, whatever you're about to discover, you'll believe that a given proposition is false, then don't currently believe that the proposition is true:
\begin{itemize}
    \item[$\Pi R$] If $\Pi$ is a partition any member of which you could discover, there is a $p\in \Pi$ such that you shouldn't reverse any of your beliefs upon discovering $p$.
   \begin{itemize}
       \item[] If $\Pi\subseteq \mathcal{E}$ is a partition of $E$, then $B(E) \cap  B(E\cap p)\neq \emptyset$ for some $p\in \Pi$.
   \end{itemize} 
\end{itemize}
By contrast, if belief requires probability over a threshold greater than $.5$ (as it does on our account), this principle cannot fail.\footnote{\label{fn:weakbelief}Failures of $\Pi R$ are to be expected for certain notions of belief that are weaker than the one we are operating with here. For example, your `best guess' about what deck you are holding plausibly does change no matter what card you draw; and arguably what we `believe' (in ordinary English) often aligns with our best guesses. See \cite{Holg22} and \cite{DorsMandfc} for discussion.}

Overall, then, we see few advantages for the tracking theory over our own. Given \textsc{orthogonality}, which Lin and Kelly essentially build into their formalism, the tracking theory offers a stronger theory of belief revision. However, the theory violates \textsc{threshold}, often in dramatic ways. Moreover, to make reasonable predictions in cases like \textbf{One Hundred Flips}, both theories need to appeal to coarse-grained questions that conflict with \textsc{orthogonality}. Having done so, both theories invalidate many principles of belief revision, although the details differ slightly (with our theory invalidating $\Box +$ and Lin and Kelly's invalidating $\Pi R$).

\section{Further work}

We conclude with three directions for further work. One concerns nonmonotonic consequence, where $p\nc q$ is interpreted as $B(p)\subseteq q$. We think that distinguishing one's evidence from one's beliefs that go beyond one's evidence offers a productive way of thinking about nonmonotonic consequence, and that the logic resulting from our framework contrasts in interesting ways with the one resulting from Lockean theories of belief (explored in \cite{Hawt96}).  

The second direction concerns constraints on $\mathcal{E}$. Consider, for example, the Monty Hall problem, in which it is crucial that when one gets new evidence about one's environment, one also gets evidence that one has gotten such evidence. We argue in \cite{GoodSaloRev} that such cases motivate a \emph{nestedness} requirement on $\mathcal{E}$: if two possible bodies of evidence are mutually consistent, then one entails the other. This requirement induce new subtleties in the resulting nonmonotonic logic. 

A third question for future work concerns what happens when probability structures are generalized by making the relevant question a function of one's evidence.  \cite[Appendix~C]{GoodSaloTARK} motivate this generalization, in order to vindicate certain judgments about a family of examples discussed in \cite{GoodSaloEpis}. We hope to explore these models in future work; one notable feature is that they invalidate $\Box -$ but still validate $\Pi R$. 

\section*{Acknowledgements}

We thank Kevin Dorst, Josh Pearson, and three anonymous referees for TARK for very helpful comments on earlier versions of this material.

\appendix

\section{Proofs}

\setcounter{proposition}{0}

\begin{proposition}
$\Box -$ and $\Box R$ are valid on the class of probability structures. 
\end{proposition}
\begin{proof}
Since $\Box -$ entails $\Box R$, it's sufficient to prove the former. We suppose that $B(E) \subseteq p$, and show that $B(E\cap p)\subseteq B(E)$.

Note that  if $s\in B(E)$, $[s]_Q\subseteq B(E)\subseteq p$. So for any $q\in Q$, if $Pr([s]_Q|E)\geq Pr(q|E)$, then also $Pr([s]_Q|E\cap p)\geq Pr(q|E\cap p)$. Contraposing, this means that if $Pr(q|E\cap p)\geq Pr([s]_Q|E\cap p)$ and $s\in B(E)$, then $Pr(q|E)\geq Pr([s]_Q|E)$, and so $q \subseteq B(E)$.

Moreover, since $B(E) \subseteq p$, $Pr(B(E)|E\cap p)\geq Pr(B(E)|E)\geq t$.

Now, note that $B(E\cap p)$ is the minimal $X\subseteq E\cap p$ such that (i) if $s \in X$ and $Pr(q|E\cap p)\geq Pr([s]_Q|E\cap p)$ for $q\in Q$, then $q \subseteq X$, and (ii) $Pr(X|E\cap p)\geq t$. By the above, $B(E)$ satisfies both (i) and (ii); so it contains the minimal such $X$ as a subset. So $B(E\cap p)\subseteq B(E)$, as required.

\end{proof}

\begin{proposition}
$\Diamond -$, $\Diamond R$, and $\Box +$ can all fail in  probability structures.
\end{proposition}
\begin{proof}Counter-models are given in the main text.
\end{proof}

\begin{proposition}
$\Diamond R$ is valid in  probability structures satisfying \textsc{orthogonality}. 
\end{proposition}

\begin{proof}
Suppose that $B(E)\cap p \neq \emptyset$. Let $s\in B(E)\cap p$ be such that $Pr_E([s]_Q)\geq Pr_E([s']_Q)$ for any $s' \in B(E)\cap p$. We will show that, given \textsc{orthogonality}, there can be no $q\in Q$ such that $Pr_{E\cap p}(q)> Pr_{E\cap p}([s]_Q)$. It follows that $s \in B(E \cap p)$, thus establishing $B(E)\cap B(E\cap p)\neq \emptyset$.

By \textsc{orthogonality}, if $q\in Q$ and $Pr_{E\cap p}(q)> Pr_{E\cap p}([s]_Q)$, then either $Pr_E(q)> Pr_E([s]_Q)$ or else $Pr_{E\cap p}([s]_Q)=0$. But $s\in E\cap p$, so $Pr_{E\cap p}([s]_Q)\neq 0$. So suppose $Pr_E(q)> Pr_E([s]_Q)$. By the way $s$ was chosen, it follows that $q\cap (B(E)\cap p) = \emptyset$. But $q\cap p \neq \emptyset$, since $Pr_{E\cap p}(q)>0$. So $q\cap B(E) = \emptyset$. But since $s \in B(E)$, this contradicts the assumption that $Pr_E(q)> Pr_E([s]_Q)$.

\end{proof}

\begin{proposition}
$\Pi -$ can fail in probability structures. But it is valid in  probability structures satisfying \textsc{orthogonality}.
\end{proposition}

\begin{proof}
To see that $\Pi -$ can fail, consider 
\begin{quote}
    \begin{tabular}{ll}
$S=\{s_1,s_2,s_3,s_4,s_5,s_6\}$ & $\mathcal{E}=\{S,\{s_1,s_3,s_5\},\{s_2,s_4,s_6\}\}$ \\ $Q=\{\{s_1,s_2\},\{s_3,s_4\},\{s_5\},\{s_6\}\}$  & $Pr$ is uniform \\
$t=.65$
\end{tabular}
\end{quote}
Let $p=\{s_1,s_3,s_5\}$ and $\Pi=\{p, S\setminus p\}$. Then $B(S\cap p)=\{s_1,s_3,s_5\}\not \subseteq \{s_1,s_2,s_3,s_4\} = B(S)$ and $B(S\cap S\setminus p)=\{s_2,s_4,s_6\}\not \subseteq B(S)$.

Now suppose $\langle S,\mathcal{E},Q,Pr,t\rangle$ satisfies \textsc{orthogonality}. To show that $\Pi -$ holds, we suppose that $B(E\cap p_i) \not\subseteq B(E)$ for each $p_i\in\Pi$, and deduce a contradiction.

For each $i$, let $s_i\in B(E\cap p_i)\setminus B(E)$ be such that $Pr_{E\cap p_i}([s_i]_Q)\geq Pr_{E\cap p_i}([s]_Q)$ for every $s \in  B(E\cap p_i)\setminus B(E)$. Since $s_i \in B(E\cap p_i)$, $Pr_{E\cap p_i}(\{s:Pr_{E\cap p_i}([s]_Q)\geq Pr_{E\cap p_i}([s_i]_Q)\} < t$. By \textsc{orthogonality}, $Pr_E([s]_Q)\geq Pr_E([s_i]_Q)$ entails that either $Pr_{E\cap p_i}([s]_Q)\geq Pr_{E\cap p_i}([s_i]_Q)$ or $Pr_{E\cap p_i}([s]_Q)=0$. So $Pr_{E\cap p_i}(\{s:Pr_E([s]_Q)\geq Pr_E([s_i]_Q)\})=Pr_{E\cap p_i}(\{s:Pr_{E\cap p_i}([s]_Q)\geq Pr_{E\cap p_i}([s_i]_Q)\})< t$.


Now let $k$ be such that, for every $i$, $Pr_E([s_k]_Q)\geq Pr_E([s_i]_Q)$. Then $\{s:Pr_E([s]_Q)\geq Pr_E([s_k]_Q)\} \subseteq \{s:Pr_E([s]_Q)\geq Pr_E([s_i]_Q)\}$, and so $Pr_{E\cap p_i}(\{s:Pr_E([s]_Q)\geq Pr_E([s_k]_Q)\})\leq Pr_{E\cap p_i}(\{s:Pr_E([s]_Q)\geq Pr_E([s_i]_Q)\})<t$ for every $i$. But then by the law of total probability, $Pr_{E}(\{s:Pr_E([s]_Q)\geq Pr_E([s_k]_Q)\})<t$, contradicting the assumption that $s_k \notin B(E)$.

\end{proof}

\begin{proposition}
$\Diamond -$ is valid in probability structures satisfying \textsc{stability} and \textsc{orthogonality}; but $\Box+$ can fail in such structures. Moreover, $\Diamond -$ can fail in probability structures satisfying \textsc{stability} in which \textsc{orthogonality} fails.
\end{proposition}

\begin{proof} To see that $\Diamond -$ holds, note that $B(E\cap p)$ is the minimal $X\subseteq E\cap p$ such that (i) if $s \in X$ and $Pr(q|E\cap p)\geq Pr([s]_Q|E\cap p)$ for $q\in Q$, then $q \subseteq X$, and (ii) $Pr(X|E\cap p)\geq t$. Then if $B(E)\cap p \neq \emptyset$, $Pr(B(E)\cap p|E \cap p)=Pr(B(E)|E\cap p)\geq t$ by \textsc{stability}, so $B(E)\cap p$ meets condition (ii). Moreover, it meets condition (i) by \textsc{orthogonality}. So $B(E)\cap p$ contains the minimal $X$ meeting (i) and (ii) as a subset. So $B(E\cap p)\subseteq B(E)\cap p \subseteq B(E)$, as required.


To see how $\Box+$ can fail, let $S= \{a,b,c\}$,  $\mathcal{E}=\{S,\{a,b\}\}$, $Q=\{\{s\}:s\in S\}$,  $Pr(\{a\})=.9$, $Pr(\{b\})=.09$, $Pr(\{c\})=.01$, and $t=.9001$. This structure satisfies \textsc{stability}. $\Box+$ fails, since $B(S)=\{a,b\}\not\subseteq B(\{a,b\})=\{a\}$.

To see how $\Diamond-$ can fail in the absence of \textsc{orthogonality}, consider a probability structure in which $Q=\{A,B,C\}$, $\mathcal{E}=\{S,E\}$, $Pr(A)=\frac{1}{2}, Pr(B)=\frac{1}{4}+\epsilon, Pr(C) = \frac{1}{4}-\epsilon$, $Pr_E(A)=Pr_E(B)=Pr_E(C)=\frac{1}{3}$, and $t=\frac{1}{2}+\epsilon$.  \textsc{stability} hold, but $\Diamond-$ fails: $B(S)\cap E\neq\emptyset$, but  $B(E)=E\not\subseteq B(S)= A\cup B$.  
\end{proof}

\begin{proposition}
$\Box +$, $\Box -$, $\Box R$, $\Diamond R$, and $\Pi -$ are valid for LK-belief on the class of probability structures satisfying \textsc{orthogonality} 
\begin{proof} For $\Diamond R$, see the proof of Proposition 3; For $\Box +$ and $\Box -$ (and hence $\Box R$), see the proof of Proposition 7; for $\Pi -$, see \cite{GoodSaloRev}. 
\end{proof}
\end{proposition}

\begin{proposition}
$\Box +$ and $\Box -$ are valid for LK-Belief on the class of probability structures. $\Diamond R$ and $\Pi -$ can both fail in such structures.
\end{proposition}

\begin{proof}
The failures of $\Diamond R$ and $\Pi -$ follow from the failure of $\Pi R$ described in the main text.

Suppose that $B_{LK}(E)\subseteq p$. We will show that $B_{LK}(E\cap p)=B_{LK}(E)$, thus establishing $\Box +$ and $\Box -$.

Since $B_{LK}(E)\subseteq p$, we have that if $s \in B_{LK}(E)$, then $[s]_Q\subseteq p$. So if $s \in B_{LK}(E)$ then $\frac{Pr([s]_Q|E\cap p)}{Pr(q|E\cap p)}\geq \frac{Pr([s]_Q|E)}{Pr(q|E)}$ for any $q\in Q$ such that $Pr(q|E\cap p)>0$. So if $s\in B_{LK}(E)$, then for any $q\in Q$ with $Pr(q|E\cap p)>0$, $\frac{Pr([s]_Q|E\cap p)}{Pr(q|E\cap p)}\geq \frac{Pr([s]_Q|E)}{Pr(q|E)} \geq t$. And if $Pr(q|E\cap p)=0$, then trivially $Pr([s]_Q|E)\geq t\times Pr(q|E\cap p)$. So $s\in B_{LK}(E \cap p)$. 

Moreover, if $s\notin  B_{LK}(E)$, then $t>0$ and there is a $q\in Q$ such that (i) $Pr(q|E) \times t > Pr([s]_Q|E)$ and (ii) $q\cap B_{LK}(E) \neq \emptyset$. Assuming $Pr([s]_Q|E \cap p)>0$ then, by the above, $\frac{Pr(q|E\cap p)}{Pr([s]_Q|E \cap p)} \geq \frac{Pr(q|E)}{Pr([s]_Q|E)} > \frac{1}{t}$. In that case $s\notin B_{LK}(E\cap p)$. And if $Pr([s]_Q|E \cap p)>0$ then also $t\times Pr(q|E\cap p)> \times Pr([s]_Q|E \cap p)$. So in that case also $s\notin B_{LK}(E\cap p)$.

So $B_{LK}(E\cap p)=B_{LK}(E)$, as required.


\end{proof}

\typeout{}

\bibliographystyle{eptcs} 
\bibliography{Bibliography}

\end{document}